\documentclass{article} 
\usepackage{iclr,times}

\usepackage{booktabs}
\usepackage{hyperref}
\usepackage{url}
\iclrfinalcopy
\usepackage[utf8]{inputenc}
\usepackage{amsmath,amsthm}
\usepackage{mdframed}
\usepackage{graphicx}

\usepackage{thmtools}
\usepackage{thm-restate}

\theoremstyle{definition}

\newcommand{\yuandong}[1]{\textcolor{red}{[Yuandong: #1]}}
\newcommand{\clj}[1]{\textcolor{blue}{[clj: #1]}}

\newmdenv[
  font=\ttfamily\small,
  linewidth=0.5pt,
  innerleftmargin=10pt,
  innerrightmargin=10pt,
  innertopmargin=10pt,
  innerbottommargin=10pt,
]{monobox}

\newcommand{\ignore}[1]{}

\title{Extending Context Window of Large Language Models via Position Interpolation}
\author{Shouyuan Chen \hspace{1em} Sherman Wong \hspace{1em} Liangjian Chen \hspace{1em} Yuandong Tian \\
Meta Platforms Inc. \\
\texttt{\{chenshouyuan,shermanwong,clj,yuandong\}@meta.com} 
}

\begin{document}

\maketitle

\begin{abstract}
We present Position Interpolation (PI) that extends 
the context window sizes of RoPE-based~\citep{su2021roformer} pretrained LLMs such as LLaMA~\citep{touvron2023llama} models to up to 32768 with minimal fine-tuning (within 1000 steps), while demonstrating strong empirical results on various tasks that require long context, including passkey retrieval, language modeling, and long document summarization from LLaMA 7B to 65B. Meanwhile, the extended model by Position Interpolation preserve quality relatively well on tasks within its original context window. To achieve this goal, Position Interpolation linearly down-scales the input position indices to match the original context window size, rather than extrapolating beyond the trained context length which may lead to catastrophically high attention scores that completely ruin the self-attention mechanism. Our theoretical study shows that the upper bound of interpolation is at least $\sim 600 \times$ smaller than that of extrapolation, further demonstrating its stability. Models extended via Position Interpolation retain its original architecture
and can reuse most pre-existing optimization and infrastructure.
\end{abstract}


\section{Introduction}
\ignore{ 
\clj{In recent years, large language models (LLMs) have surpassed expectations, with notable examples like Chat-GPT, GPT-4 from OpenAI, and LLaMA from Meta. These models have demonstrated impressive capabilities in neutral language processing tasks. A key factor contributing to their success is the incorporation of long context windows, enabling the models to grasp contextual information effectively.
The inclusion of long context windows empowers LLMs to excel in long-form generation tasks. This capability unlocks a wide range of practical applications, such as incorporating extended conversational context, writing or continuing novels, summarizing financial or legal reports, and providing detailed answers to important questions.}   \yuandong{We should start from why long-form language model tasks are important and list a few existing models}. 
}

Large language models (LLMs) typically come with a pre-defined context window size. For example, inputs to LLaMA models \citep{touvron2023llama} must be fewer than 2048 tokens. This pre-set context window limit is frequently exceeded in applications such as conducting long conversations, summarizing long documents, or executing long-term planning. For these applications, LLMs with longer context windows are preferred. However, training an LLM from scratch with long context windows requires significant investments. This naturally leads to a question: Can we extend the context window of an existing pre-trained LLM?

One straightforward approach is to fine-tune an existing pre-trained Transformer with a longer context window. However, empirically, we found that models trained this way adapt to long context windows very slowly. 
After training for more than 10000 batches, the effective context window saw a minimal increase, moving from 2048 to 2560 (Table~\ref{fig:passcode}). This suggests that such method is inefficient for extending to substantially longer context windows.
\ignore{
\begin{figure}[]
\centering
\includegraphics[width=0.8\textwidth]{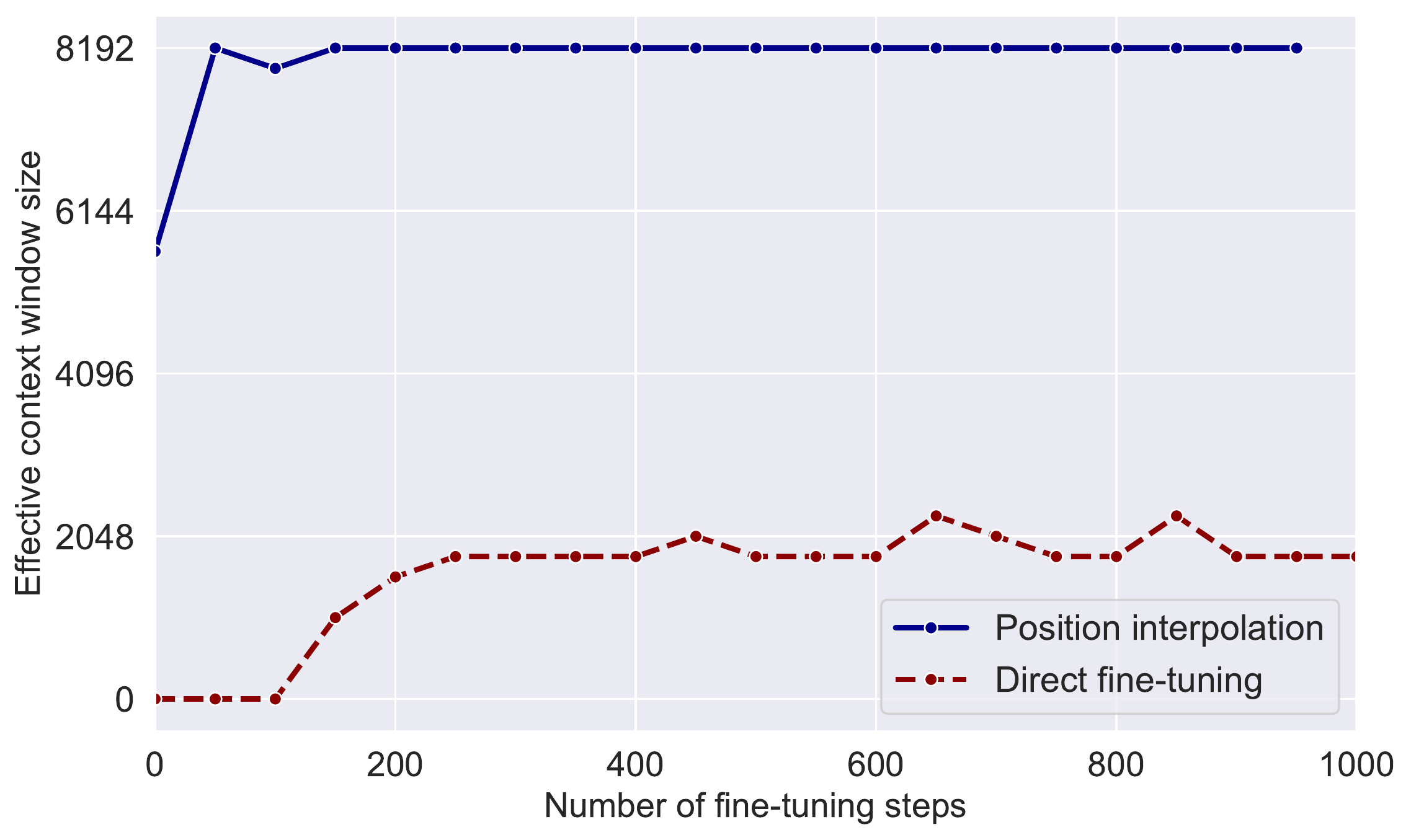}
\caption{
\small Effective context window size of LLaMA after extending using direct fine-tuning v.s. Position Interpolation. 
Direct fine-tuning method increases the effective context
window size very slowly, only to slightly more than the original 2048 after fine-tuining for 1000 steps.
Position Interpolation can enable full effective context window after tuning within 100 steps. 
See Section~\ref{sec:passkey} for the measurement method of effective context window size and Section~\ref{sec:setup} for 
training setup.
}
\label{fig:effective-window-size}
\end{figure}
}

While certain techniques such as ALiBi~\citep{press2022train} and LeX~\citep{sun2022lengthextrapolatable} enable length extrapolation of Transformers, i.e. train on short context windows and inference on longer ones, many existing pre-trained LLMs, including LLaMA~\citep{touvron2023llama}, use positional encodings that have weak extrapolation properties (e.g., RoPE~\citep{su2021roformer}). Therefore, the applicability of these techniques for extending the context window sizes of such LLMs remains limited.

In this work, we introduce Position Interpolation to enable context window extensions for certain existing pre-trained LLMs, including LLaMA. The key idea is, instead of extrapolation, we directly down-scale the position indices so that the maximum position index matches the previous context window limit in the pre-training stage. See Figure~\ref{fig:sin} for an illustration. In other words, to accommodate more input tokens, we interpolate the position encodings at neighboring integer positions, utilizing the fact that position encodings can be applied on non-integer positions, as opposed to extrapolating outside the trained positions, which may lead to catastrophic values. We verify our approach theoretically, by showing that the interpolated attention score has a much smaller upper bound ($\sim 600\times$ smaller in LLaMA 7B setting) than the extrapolated one, and is thus much more stable. Therefore, interpolated position encodings are easier for the model to adapt. 
\begin{figure}[]
\centering
\includegraphics[width=\textwidth]{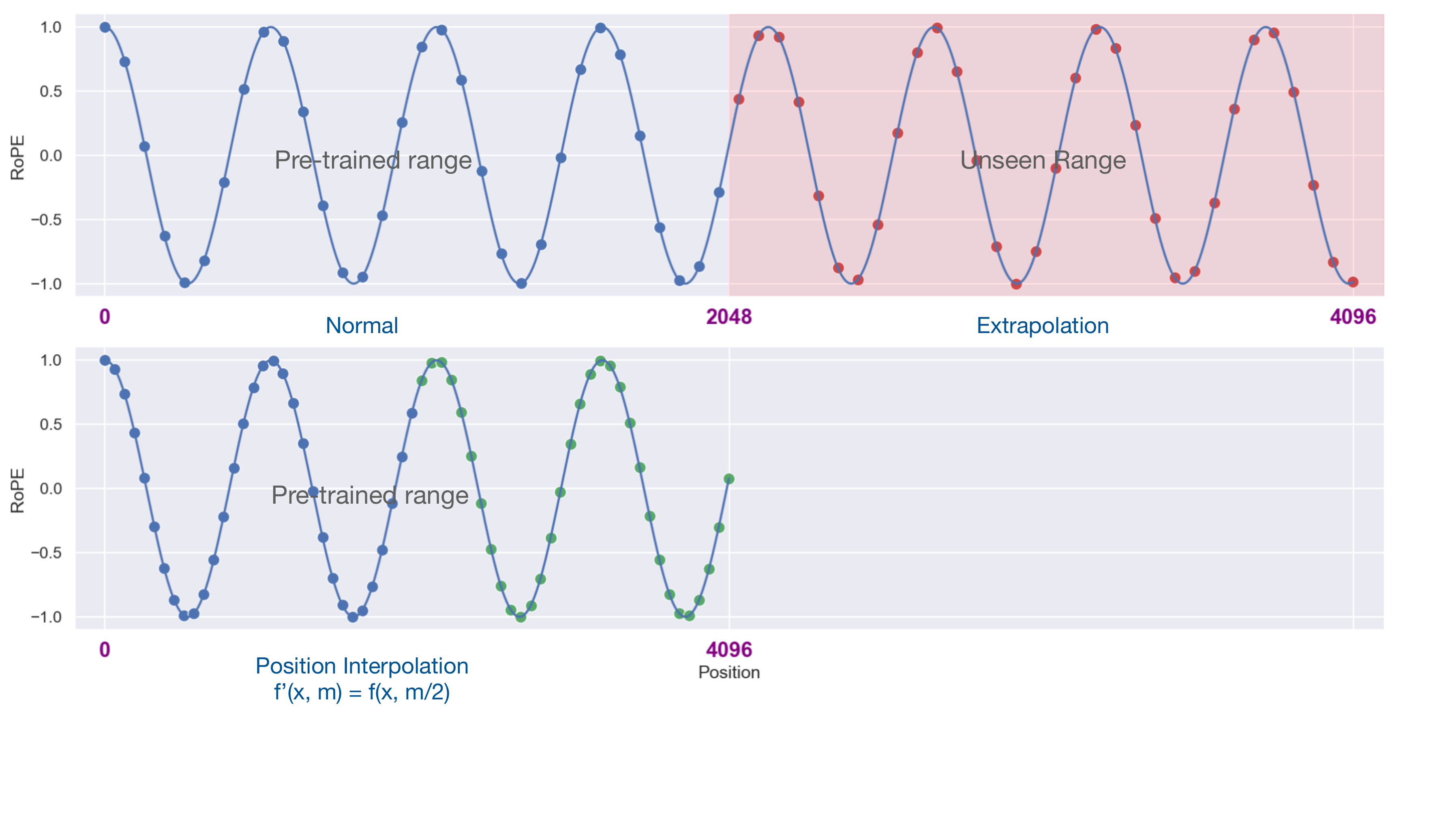}
\caption{\small
 An illustration of our Position Interpolation method. Consider a Llama model pre-trained with a 2048 context window length. Upper left illustrates the normal usage of an LLM model: input position indices (blue dots) are within the pre-trained range. Upper right illustrates length extrapolation where models are required to operate unseen positions (red dots) up to 4096. Lower left illustrates Position Interpolation where we downscale the position indices (blue and green dots) themselves from [0, 4096] to [0, 2048] to force them to reside in the pretrained range.
    }
\label{fig:sin}
\end{figure}

Empirically, we found that Position Interpolation is highly effective and efficient, requiring only a very short period of fine-tuning for the model to fully adapt to greatly extended context windows. We present experimental results for extending the context window to up to 32768 from the initial 2048 across 7B to 65B LLaMA models using Position Interpolation. Our results show that
\begin{enumerate}
    \item Position Interpolation can easily enable very long context windows (e.g. 32768), requiring only fine-tuning for 1000 steps on the Pile \citep{pile} to achieve a good quality. 
    The cost of fine-tuning is negligible compared to the pre-training costs. 
    This confirms our hypothesis that it is relatively easy for the models to adapt to interpolated position encodings.
    
    \item Position Interpolation generates strong models that can effectively make use of much extended context window. We show that models extended by Position Interpolation enjoy significant perplexity gains from greatly extended context windows for text modeling, and we show that the perplexity reduces graceful with the enlargement of context windows. We also applied Position Interpolation in a long text summarization task, and demonstrate competitive performances.
    
    \item Position Interpolation preserves model quality relatively well for tasks within its original context window sizes. We present a variety of evaluation results for the extended LLaMA models on the original LLaMA benchmark. Compared with original LLaMA models, the extended LLaMA models saw a minor degradation on several standard benchmarks within a 2048 token limit.
\end{enumerate}

Our results highlight the innate ability of Transformer models to “extrapolate to sequence lengths longer than the ones encountered during training” as hypothesized in the seminal work of \citet{transformer}. We reaffirm this hypothesis and suggest that the previously known weakness of extrapolating to longer sequences for language modeling \citep{press2022train} may be due to direct extrapolation of positional encodings and it can be largely mitigated by interpolating position encodings instead. 

\textbf{Concurrent work.} Right before our release, we are informed with a concurrent blogpost (SuperHOT~\cite{superhot}) that also interpolates positional encoding in RoPE to extend the context window from 2K to 8K. Recently, open source community picks it up in Reddit post~\footnote{\url{https://www.reddit.com/r/LocalLLaMA/comments/14fgjqj/a_simple_way_to_extending_context_to_8k/}} and Github Issues~\footnote{\url{https://github.com/ggerganov/llama.cpp/discussions/1965}}, which shows that fine-tuning with LoRA~\citep{hu2021lora} also seems to work well. Our paper shows a full fine-tuning with up to 65B model work well with Position Interpolation, and we also give theoretical explanations why interpolation achieves much more stable results than extrapolation, by showing that the upper bound of interplated attention score is much lower than that of extrapolated ones.  

\section{Method}

\def\vf{\mathbf{f}}
\def\vx{\mathbf{x}}
\def\vk{\mathbf{k}}
\def\vq{\mathbf{q}}
\def\vu{\mathbf{u}}
\def\di{\mathrm{i}}

\subsection{Background: Rotary Position Embedding (RoPE)} 
Transformer models require explicit positional information to be injected, typically in the form of positional encodings, to represent the order of inputs. We consider Rotary Position Embedding (RoPE) \citep{su2021roformer}, which is the position encoding used in the LLaMA model \citep{touvron2023llama}. Given a position index $m \in [0, c)$ and an embedding vector $\vx := [x_0, x_1, \ldots, x_{d-1}]^\top$, where $d$ is the dimension of the attention head, RoPE defines a vector-valued complex function $\vf(\vx, m)$ as follows
\begin{equation}
    \vf(\vx,m) = [(x_0 + \di x_1) e^{\di m \theta_0}, (x_2 + \di x_3) e^{\di m \theta_1}, \ldots, (x_{d-2} + \di x_{d-1})e^{\di m \theta_{d/2-1}}]^\top
\end{equation}
where $\di := \sqrt{-1}$ is the imaginary unit and $\theta_j = 10000^{-2j/d}$. Using RoPE, the self-attention score 
\begin{eqnarray}
a(m,n) &=& \mathrm{Re}\langle\vf(\vq, m), \vf(\vk, n)\rangle \nonumber \\
&=& \mathrm{Re}\left[\sum_{j=0}^{d/2-1} (q_{2j} +\di q_{2j+1})(k_{2j} - \di k_{2j+1}) e^{\di (m-n)\theta_j}\right] \nonumber \\
&=& \sum_{j=0}^{d/2-1} (q_{2j} k_{2j} + q_{2j+1}k_{2j+1})\cos((m-n)\theta_j) + (q_{2j} k_{2j+1} - q_{2j+1}k_{2j})\sin((m-n)\theta_j) \nonumber \\
&=:& a(m-n) \label{eq:attn-score}
\end{eqnarray}
is only dependent on relative position $m-n$ through trigonometric functions. Here $\vq$ and $\vk$ are the query and key vector for a specific attention head. At each layer, RoPE is applied on both query and key embeddings for computing attention scores. 

\subsection{Direct Extrapolation} 
While the attention score in RoPE only depends on the relative positions, which is what we want, its extrapolation performance is not great~\cite{}. In particular, when directly extending to larger context windows unseen in the training, the perplexity may shoot up to very high numbers (i.e., $>10^3$), comparable to untrained models.  

Ideally, we want to see the model trained on a context window of size $L=2048$ to still work reasonably well on longer context window, but may not have the capability to leverage information that appears beyond $L$. For example, to answer a question located at 3000, the model trained on maximal window size of $L=2048$ cannot leverage evidences provided at location 0, but still can leverage the evidences provided at location 2900. In contrast, in reality we see catastrophic behaviors, i.e., question at location 3000 cannot be answered correctly, even if the evidences are located at location 2900.  

\begin{figure}
    \centering
    \includegraphics[width=\textwidth]{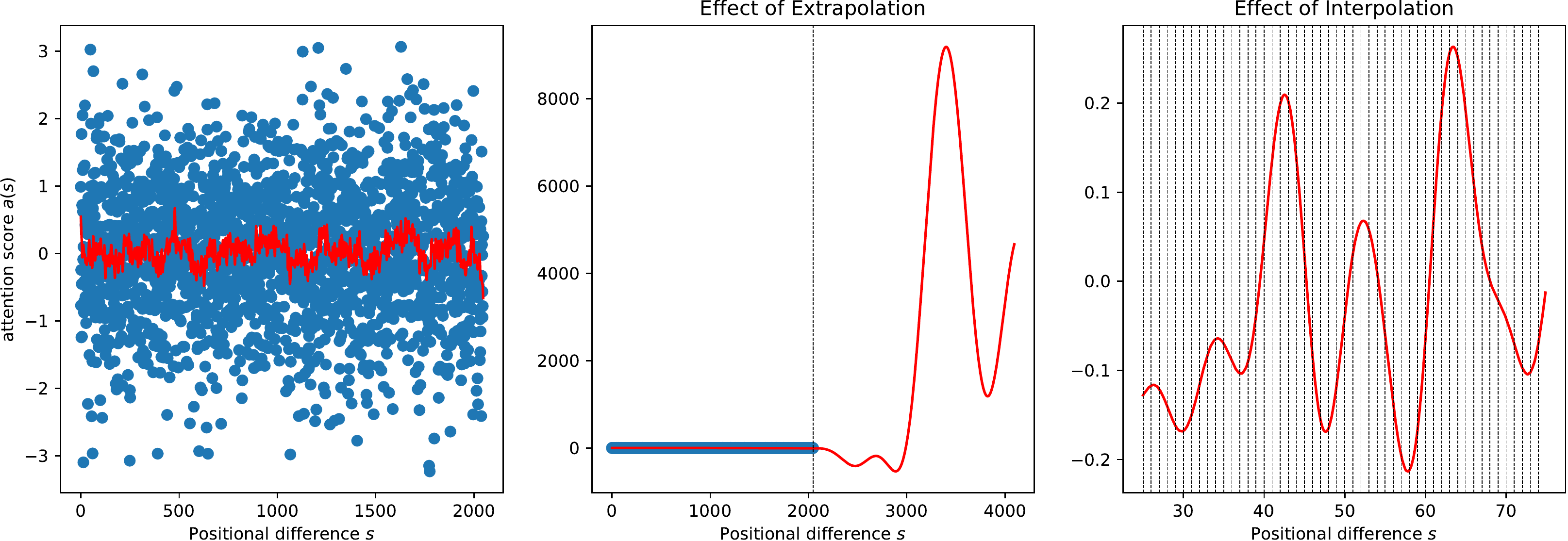}
    \caption{\small Extrapolation versus interpolation. \textbf{Left:} a fitted attention score function (in red) in the form of Eqn.~\ref{eq:attn-basis-expansion} with $d= d_\mathrm{model} / n_\mathrm{head} = 4096 / 32 = 128$ (setting of LLaMA 7B). Dots are random input points to be fitted and red curve is the fitted score function via least square, which is approximately within $[-1,1]$. \textbf{Middle:} While the fitted function seems to be well bounded in $[0,L]$, where $L = 2048$, out of this region it may goes beyond $8000$, causing catastrophic issues in attention computation. Note that here we do not cherry pick at all: almost every learned curve from a set of randomly generated input points within $[0, L]$ has the extrapolation issue. \textbf{Right:} On the other hand, interpolation is much more stable. Curves in between vertical dotted lines (i.e., integer positional difference) are smooth and well-behaved. Please check Appendix~\ref{sec:code-inter-extra} for the source code used to generate the figure.}
    \label{fig:interp-vs-extrap}
\end{figure}

What is the reason behind? How could this happen if the attention score $a_{m-n}$ decays as the relative distance $|m-n|$ increases, according to Section 3.4.3 of~\citep{su2021roformer}, and content from very far distances should not matter that much? It turns out that the upper bound derived in Section 3.4.3 of~\citep{su2021roformer} may be too loose: while it indeed decays with respect to $|m-n|$, the bound can still be quite large (i.e., the bound can be critically depends on the magnitude of $v_j$) and thus vacuous.  
In fact, if we treat all trigonometric functions as basis functions (i.e, $\phi_j(s):=e^{\di s \theta_j}$), and think about Eqn.~\ref{eq:attn-score} as basis expansion as the following:
\begin{equation}
    a(s) = \mathrm{Re}\left[\sum_{j=0}^{d/2-1} h_j e^{\di s \theta_j}\right] \label{eq:attn-basis-expansion}
\end{equation}
where $s$ is the positional span between a query and a key and $h_j := (q_{2j} +\di q_{2j+1})(k_{2j} - \di k_{2j+1})$ are complex coefficients depending on $\vq$ and $\vk$ (here the definition of $h_j$ is exactly the same as the definition of $h_j$ in Sec 3.4.3 in RoPE~\citep{su2021roformer}). Now the the issue becomes clear: as shown in Fig.~\ref{fig:interp-vs-extrap}, $a_s$ can be small in magnitude in the range of $[0, 2048]$, but gives huge values out of the region. The underlying reason is that the trigonometric family $\{\phi_j\}$ (with sufficiently large $d$) is a universal approximator and can fit any arbitrary functions. Therefore, for $a_s$, there always exist coefficients $\{h_j\}$ (i.e. key and query) that corresponds to small function values in [0, 2048] but much larger in regions beyond. 

\subsection{Proposed approach: Position Interpolation (PI)} 
In Fig.~\ref{fig:interp-vs-extrap}, thanks to the smoothness of bases functions $\phi_j$
\emph{interpolation} is much more stable and will not lead to wild values. Therefore, instead of extrapolate the attention score in Eqn.~\ref{eq:attn-basis-expansion} to $s > L$, how about we define an attention score $\tilde a(s) = a(Ls/L')$ where $L'$ is the longer context window? Formally, we replace RoPE $\vf$ by $\vf’$ defined as follows
\begin{equation}
    \vf’(\vx, m)= \vf\left(\vx, \frac{mL}{L'} \right).
\end{equation}
We call this transformation on the position encoding \textbf{Position Interpolation}. In this step, we reduce position indices from $[0, L')$ to $[0, L)$ to match the original range of indices before computing RoPE. Consequently, as inputs to RoPE, the maximum relative distance between any two tokens has been reduced from $L'$ to $L$. Since we align the ranges of position indices and relative distances before and after extension, we mitigate the effect on attention score computation due to context window extensions, which can allow the model easier to adapt. To further demonstrate this is the case, in the following theorem, we show that the interpolated attention score is well-behaved: 

\begin{restatable}[Interpolation bound]{theorem}{interp}
For attention score $a(s) = \mathrm{Re}\left[\sum_{j=0}^{d/2-1} h_j e^{\di s \theta_j}\right]$, where $\theta_j = c^{-2j/d}$, its interpolation value $a(s)$ for $s \in [s_1, s_2]$ is bounded as follows:
\begin{equation}
    |a(s) - a_{\mathrm{linear}}(s)| \le d\left(\max_j |h_j|\right) \frac{(s-s_1)(s_2-s)}{8\ln c} \label{eq:interp-bound}
\end{equation}
where $a_{\mathrm{linear}}(s)$ is the linear interpolation of two grid point $a(s_1)$ and $a(s_2)$ that are known to behave well, enforced by LLM pre-training:
\begin{equation}
    a_{\mathrm{linear}}(s) := (1-\lambda(s)) a(s_1) + \lambda(s) a(s_2), \quad\quad \lambda(s) := \frac{s-s_1}{s_2-s_1}  
\end{equation}
\end{restatable}
Please check Appendix~\ref{sec:proof} for the proof. Intuitively, in LLM pre-training, we know that the attention score $a(s)$ behaves well on integer grid $s_1$ and $s_2$. Therefore, for any interpolation $s\in [s_1,s_2]$, we have $(s-s_1)(s_2-s) \le 1/4$. Note that $c = 10000$, the bound becomes:
\begin{equation}
    |a(s) - a_{\mathrm{linear}}(s)| \le \frac{d}{32 \ln c}\max_j |h_j| \approx \frac{d \max_j |h_j|}{294.73}    
\end{equation}
In comparison, Sec. 3.4.3 in RoPE~\citep{su2021roformer} yields an extrapolation bound (i.e., it works for all positional distance $s$):
\begin{equation}
 |a(s)| \le \left(\max_{j} |h_j - h_{j+1}|\right) \sum_{k=0}^{d/2-1} |A_{k+1}(s)| \le 2\left(\max_{j} |h_j|\right) \sum_{k=0}^{d/2-1} |A_{k+1}(s)| \label{eq:extra-bound}, 
\end{equation}
where $A_k(s) := \sum_{j=0}^{k-1} e^{\di s \theta_j}$. While there is no close form for $B(s) := \sum_{k=0}^{d/2-1} |A_{k+1}(s)|$, numerically it is at least larger than $d$, and for many positional difference $s$, $B(s)$ is much larger than $d$ (check Appendix~\ref{sec:visualization} for the plot). Therefore, the interpolation bound is at least $2 \cdot 294.73\sim 600\times$ smaller than the extrapolation bound, and thus the interpolated attention score is much more stable than extrapolated one.  

Notably, our method of rescaling of position indices does not introduce extra weight, or modify the model architecture in any way. This makes it attractive in practical applications, since most infrastructure and optimization for the original model can be reused after the extension.  

\ignore{[Something something resolution, optional] Our method accommodates more input tokens by interpolating the position encoding in the original range of $[0, c)$, resulting in a more densely distributed position encoding. As our experimental results indicate, Transformer models demonstrate a remarkable ability to handle vastly denser position encodings that come with significantly larger context windows without much fine-tuning. }

{\noindent \textbf{Fine-tuning.}} We can further fine-tune the interpolated model using the next token prediction task with interpolated position encodings on the extended context window size using a pre-training corpus such as the Pile \citep{pile}. In the next section, we show that our fine-tuning process only needs tens to hundreds thousands of examples.  We also find that the result of the fine-tuning is not sensitive to the choice of examples. The reason may be that the model is only adapting to the new context window during the fine-tuning phase, starting from a good initialization, as opposed to acquiring new knowledge. 

{\noindent \textbf{Other ways to reduce interpolation/extrapolation bound.}} From the expression of the interpolation (Eqn.~\ref{eq:interp-bound}) and extrapolation bound (Eqn.~\ref{eq:extra-bound}), a common term is $\max_j |h_j|$, which is the maximal magnitude of query/key products. If we enforce a regularization on  $|h_j|$ during LLM training, it is possible that the catastrophic extrapolation error can be mitigated or even resolved. In fact, if we apply ridge regression with proper regularization to fit a curve in Fig.~\ref{fig:interp-vs-extrap}, the magnitude of extrapolated $a(s)$ when $s > L$ can be comparable to that within $[0,L]$. To our knowledge, we are not aware of existing LLM pre-training techniques that leverage this regularization and will leave it for future work. 

\section{Experiments}

We show Position Interpolation can effectively extend context window up to 32 times of the original size, and such extension can be done 
with only several hundreds of training steps.
We show the resulting models are strong LLMs with fully effective long context windows. 
We demonstrate its performance in a number of tasks including language modeling, passkey retrieval, and 
long document summarization.
We also present benchmark results of the extended models on the original LLaMA evaluation benchmarks.

\subsection{Setup}
\label{sec:setup}

\textbf{Model Variants}. We extended the pre-trained 7B, 13B, 33B and 65B LLaMA models \citep{touvron2023llama} to various context window of sizes up to 32768,
using either direct fine-tuning or Position Interpoloation method.
Except for rescaling the position indices for models extended with Position Interpolation, we did not 
modify LLaMA model architectures \citep{touvron2023llama} in any ways.

\textbf{Training Procedure}. We fine-tune all model variants using the next token prediction objective.
We use AdamW \citep{loshchilov2018decoupled} with $\beta_1=0.9$ and $\beta_2=0.95$. 
We use a linear learning rate warmup of 20 steps starting from $10\%$ of the maximum learning rate.
For 7B and 13B models, we set the learning rate to $2\times 10^{-5}$ and for 33B and 65B models we set the learning rate to $10^{-5}$.
We set the weight decay to zero. 
For extending 7B, 13B and 33B models to the 8192 context window size, we use 32 A100 GPUs and 64 global batch size.
For all other cases we use 128 A100 GPUs and 128 global batch size.
We note that the main need of using more GPUs is memory limitation during fine-tuning, and it is possible
to use fewer GPUs in certain cases.
We train all models using PyTorch \citep{pytorch} with Fully Sharded Data Parallel \citep{zhao2023pytorch} and Flash Attention \citep{dao2022flashattention}. 

If not specified otherwise, for the Position Interpolation method, we fine-tune the models for 1000 steps. 
For the direct fine-tuning method, we use 10000 steps. 
We primarily fine-tune using the Pile training dataset \citep{pile}. In Section~\ref{sec:llama} we also compared
fine-tuning performance on the RedPajama dataset \citep{together2023redpajama}.

\subsection{Long Sequence Language Modeling}
\label{sec:perplexity}

We evaluate the long sequence language modeling performance of our extended models and baselines on two datasets: book corpus (PG-19) \citep{Rae2020Compressive} and cleaned Arxiv Math proof-pile dataset \citep{proofpile}. 

We use the test splits of PG19 \citep{Rae2020Compressive} and proof-pile \citep{proofpile}. 
For PG19, we use the whole test split consisting of 100 documents. 
For the proof-pile dataset, we use a random subsample of 128 documents with at least 32768 SentencePiece \citep{kudo-richardson-2018-sentencepiece} tokens and truncate to the first 32768 tokens for each test document. 
We evaluate perplexity at various context window size by using a sliding window approach following \citet{press2022train} with stride $S=256$.

In Table~\ref{table:perplexity-pg19} and Table~\ref{table:perplexity-pile}, we report the perplexity results for our models and baselines on the datasets. From the results, we found that models extended with our method enjoy a significantly 
improved perplexity from longer context window sizes. 
By increasing the context window size from 2048 to 16384, we observed -0.28 and -0.5 reductions of perplexity for extending LLaMA 7B models on both datasets,
-0.27 and -0.48 reductions for extending LLaMA 13B models, and -0.14 and -0.42 reductions for extending LLaMA 33B models. 
For LLaMA 65B models, we observed -0.12 and -0.3 reductions of perplexity by extending to the 8192 context window size.

In general, we observed a consistent trend of our models achieving better perplexity with
longer context windows.  
This indicates our models can effectively make use of the longer context windows to
better predict next tokens in language modeling tasks.
Moreover, we found this trend extends to 32768 window size without diminishing on the PG19 dataset for LLaMA 7B and 13B models. 
This indicates that our method may enable extension to even longer context windows.

In contrast, we observed that models extended via the direct fine-tuning method has shown 
regression (up to +0.48) or minor improvement (up to -0.12) on the perplexity at longer context windows.
This indicates that models extended this way have limited capability of making
use of context windows longer than their pre-trained settings.

We saw a minor degradation of the perplexity on the original context window of 2048 for our extended models in some cases.
For example, on the Proof-pile dataset, we saw a degradation ranging from 0.01 to 0.05 across all models with extended with Position Interpolation. 
A small degradation of performance within original evaluation context window is expected since Position Interpolation forces position encodings in original context window to reside in a much  narrower region, which may negatively affect the language model's performance.
We present more benchmark results on the original context window size in Section~\ref{sec:llama}.

In Table~\ref{tab:perplexity-finetune} we report the relationship between perplexity and the number of fine-tuning steps
for LLaMA 7B model extending to 8192 and 16384 context window sizes using Position Interpolation evaluated
on the PG19 dataset.
We can see without fine-tuning (at step 0) the model can exhibit certain language modeling capability, as 
indicated by $<20$ perplexity for extending to 8192 context window (in contrast, the direct extrapolation method leads to $>10^3$ perplexity).
With fine-tuning, we observed that the perplexity improves quickly. 
At 200 steps the models surpassed the original model's perplexity on 2048 context window size, indicating the models
gaining ability of effectively using sequences longer  than the pre-training settings for language modeling.
At 1000 steps, we can see the models have improved steadily and achieve a significantly better perplexity.

\begin{table}[thbp]
\centering
\begin{tabular}{cccccccc}
\toprule
\multicolumn{3}{c}{Model} & \multicolumn{5}{c}{Evaluation Context Window Size}\\
Size &  Context Window & Method &  2048 & 4096 & 8192 & 16384 & 32768 \\
\midrule
   7B &            2048 &   None &  7.20 &    $>10^3$ &  $>10^3$ &    $>10^3$ &     $>10^3$ \\
   7B &            8192 &     FT &  7.21 & 7.34 & 7.69 &     - &     - \\
\midrule
   7B &            8192 &     PI&  7.13 & 6.96 & 6.95 &     - &     - \\
   7B &           16384 &     PI &  7.11 & 6.93 & 6.82 &  6.83 &     - \\
   7B &           32768 &     PI&  7.23 & 7.04 & 6.91 &  6.80 &  6.77 \\
\midrule
   13B &            2048 &   None &  6.59 &    - &  - &    - &     -\\
   13B &            8192 &     FT &  6.56 & 6.57 & 6.69 &     - &     - \\
\midrule
   13B &            8192 &     PI&  6.55 & 6.42 & 6.42 &     - &     - \\
   13B &           16384 &     PI &  6.56 & 6.42 & 6.31 &  6.32 &     - \\
   13B &           32768 &     PI&  6.54 & 6.40 & 6.28 &  6.18 &  6.09 \\
\midrule
  33B &            2048 &   None &  5.82 &    - &    - &     - &     - \\
  33B &            8192 &     FT &  5.88 & 5.99 & 6.21 &     - &     - \\
\midrule
  33B &            8192 &     PI &  5.82 & 5.69 & 5.71 &     - &     - \\
  33B &           16384 &     PI &  5.87 & 5.74 & 5.67 &  5.68 &     - \\
\midrule
65B & 2048 & None & 5.49 & - & - & - & -\\
\midrule
65B & 8192 & PI & 5.42 & 5.32 & 5.37 & - & - \\  
\bottomrule
\end{tabular}
\caption{\small Evaluation perplexity on PG19 dataset \citep{Rae2020Compressive}. FT: Direct Fine-tuning. PI: Position Interpolation. Model fine-tuned with PI shows progressively lower perplexity with longer context window, showing that PI can leverage long context well, while the perplexity of FT increases over longer window. Note that overall the perplexity is higher compared to Table~\ref{table:perplexity-pile} since PG19 has very different writing styles.}
\label{table:perplexity-pg19}
\end{table}

\begin{table}[thbp]
\centering
\begin{tabular}{cccccccc}
\toprule
\multicolumn{3}{c}{Model} & \multicolumn{5}{c}{Evaluation Context Window Size}\\
Size & Context Window & Method & 2048 & 4096 & 8192 & 16384 & 32768 \\
\midrule
7B & 2048 & None & 2.77 & - & - & - & -\\
7B & 8192 & FT &  2.85 & 2.74 & 2.73 & - & -\\
\midrule
7B & 8192 & PI & 2.79 & 2.57 & 2.39 & - & -\\
7B & 16384 & PI & 2.79 & 2.57 & 2.37 & 2.25 & -\\
7B & 32768 & PI & 2.82 & 2.59 & 2.39 & 2.24 & 2.48 \\
\midrule
13B & 2048 & None & 2.66 & - & - & - & -\\
13B & 8192 & FT &  2.71 & 2.56 & 2.50 & - & -\\
\midrule
13B & 8192 & PI & 2.67 & 2.47 & 2.30 & - & -\\
13B & 16384 & PI & 2.68 & 2.47 & 2.29 & 2.18 & -\\
13B & 32768 & PI & 2.68 & 2.46 & 2.28 & 2.15 & 2.35 \\
\midrule
33B & 2048 & None & 2.49 & - & - & - & -\\
33B & 8192 & FT &  2.56 & 2.48 & 2.47 & - & -\\
\midrule
33B & 8192 & PI & 2.50 & 2.32 & 2.18 & - & - \\
33B & 16384 & PI & 2.53 & 2.34 &2.18 & 2.07 & - \\
\midrule
65B & 2048 & None & 2.42 & - & - & - & -\\
\midrule
65B & 8192 & PI & 2.43 & 2.26 & 2.12 & - & - \\
\bottomrule
\end{tabular}
\caption{\small Evaluation perplexity on Arxiv Math Proof-pile dataset \citep{proofpile}. FT: Direct Fine-tuning. PI: Position Interpolation.}
\label{table:perplexity-pile}
\end{table}

\begin{table}[thbp]
\centering
\begin{tabular}{cccccccc}
\toprule
\multicolumn{2}{c}{Model} & \multicolumn{6}{c}{Number of fine-tuning steps}\\
Size &  Context Window  &     0 &  200 &  400 &  600 &  800 &  1000 \\
\midrule
  7B &                 8192 & 16.10 & 7.12 & 7.10 & 7.02 & 6.99 &  6.95 \\
  7B &                 16384 & 112.13 & 7.05 & 6.93 & 6.88 & 6.84 &  6.83 \\
\bottomrule
\end{tabular}
\caption{\small 
Evaluation perplexity on PG19 dataset \citep{Rae2020Compressive} with respect to the number of fine-tuning steps using Position Interpolation.
}
\label{tab:perplexity-finetune}
\end{table}

\subsection{Measuring Effective Context Window Size through Passkey Retrieval}
\label{sec:passkey}
We study the effective context window size, i.e. the maximum distance of a token can \emph{effectively} attend to during inference,  of our models after extension. 
To measure this, we follow a synthetic evaluation task of passkey retrieval proposed by \citet{mohtashami2023landmark}.
In this task, the models are asked to recover a random passkey hidden in a long document. 
See Figure~\ref{fig:passcode-prompt} for the format of the document.

Given a language model, we estimate the upper and lower bounds of effective context windows as follows.
Suppose the random passkey is $k$ tokens away from the end of the input. 
When a model persistently fails to retrieve the correct passkey value across several independent attempts, 
it suggests that the effective context window size of the model is less than $k$. 
Conversely, if a model consistently succeeds in retrieving the correct passkey value, 
we deduce that the effective context window size of the model is at least $k$. 

We evaluate the 7B and 33B LLaMA model variants that are extended via Position Interpolation or direct fine-tuning. 
For each model, we use 32 different $k$ uniformly spaced in the targeted context window $L'$ and run the above tests 
for 10 times for each $k$, where each time a random passkey of 5 random digits is used. 
In Table~\ref{fig:passcode}, we report $k_{\max}$ as a function of the number of fine-tuning steps, where $k_{\max}$ is defined as the maximum $k$ such that, for all  $k' \le k$, the model has a success rate of at least 20\% on $k'$.

We can see that models extended via Position Interpolation all successfully attain their desired extension objectives in terms of effective context window sizes, indicating by the effective context window size reaching maximum $k_{\max}=L'$, after merely fine-tuning for 200 steps, consistently across both 7B and 33B model sizes and up to 32768 context windows.
In contrast, LLaMA models that are extended via direct fine-tuning only saw a minimal increase of the effective context window size $k_{\max}$ from 2048 to 2560, even after
fine-tuning for more than 10000 steps, with no clear indication of an acceleration in the increase of window size. 

\ignore{
This result highlights the effectiveness of our methods: Position Interpolation 

This also suggested the inefficiency of direct fine-tuning for extending LLaMA models to 
much longer context windows.
}

\begin{table}[th]
    \centering
\begin{tabular}{lclrrllll}
\toprule
    \multicolumn{3}{c}{Model} & \multicolumn{6}{c}{Fine-tuning steps}\\
  Size &  Context Window & Method &    200 &   400 &    600 &    800 &    1000 &   10000 \\
\midrule
 7B &  8192 &   FT &  1792 &    2048 &   2048 &  2048  & 2304 & 2560 \\
 33B &  8192 &   FT &  1792 &    2048 &   1792 &  2048  & 2304 & - \\
\midrule
 7B &  8192 &   PI &  8192 &  8192 &8192 &        8192 &       8192 &       - \\
 7B & 16384 &   PI  & 16384 & 16384 &  16384 &    16384 &       16384 &       - \\
 7B & 32768 &   PI & 32768 & 32768 &  18432 &  32768 &       32768 &       - \\
33B &  8192 &   PI &  8192 &  8192 &   8192 &   8192 &       8192 &       - \\
33B &  16384 &   PI&  16384 &  16384 &   16384 &   16384 &       16384 &       - \\
\bottomrule
\end{tabular}
    \caption{\small Effective context window sizes after fine-tuning. FT: Direct fine-tuning. PI: Position Interpolation.}
    \label{fig:passcode}
\end{table}

\begin{figure}[th]
    \small
	\texttt{\noindent
		There is an important info hidden inside a lot of irrelevant text. Find it and memorize them. I will quiz you about the important information there.\\
		The grass is green. The sky is blue. The sun is yellow. Here we go. There and back again. {\color{blue}{(repeat X times)}} \\
		The pass key is {\color{red}{12345}}. Remember it. {\color{red}{12345}} is the pass key.\\
		The grass is green. The sky is blue. The sun is yellow. Here we go. There and back again. {\color{blue}{(repeat Y times)}} \\
		What is the pass key? The pass key is\\
	}
	\caption{\small Prompt format for passkey retrieval. We use the exact same prompt as proposed by \citet{mohtashami2023landmark}. Here the passkey 12345 is replaced with a random 5-digit numbers during test.}
	\label{fig:passcode-prompt}
\end{figure}

\subsection{Benchmarks on Original Context Window Size}
\label{sec:llama}

We evaluate the models extended by Position Interpolation on several standard benchmark tasks within the original context window size of 2048.
The evaluation results are listed in Table~\ref{tab:llama_original}.
From the results, we saw that models extended to 8192 produce comparable results on 
the original benchmark which is designed for a much smaller context window, with
a degradation of up to 2\% on the benchmark tasks, for both 7B and 33B model sizes. 
Models extended to longer context windows regressed more on the benchmarks, but still in reasonable ranges for most tasks. 
We also note that the choice of fine-tuning datasets does not seem to lead significant difference in the
benchmark performances, which may be due to the limited number of fine-tuning steps used in our method.
The regression on benchmark tasks is consistent with our observation on perplexity regression in Section~\ref{sec:perplexity}.

\ignore{
There are two sources w.r.t performance degradation: 
\begin{enumerate}
\item Performance regression while aligning to specific fine-tuning dataset, this is also observed in \cite{ouyang2022training} which can be mitigated by ``rehearsing" on pre-trained dataset during finetuning. As shown in Table~\ref{tab:llama_original}, we observe
less regression when fine-tuning on Redpajama dataset which is more aligned with orignal LLaMA training set. We notice the 
performance degradation worsened when extending to longer context windows. 

\item Performance regression due to position interpolation. This is the price we pay due to loss of language model's ``resolution". We observed worse benchmark results when interpolating (extending to) longer context windows -- this is also observed in Passkey Retrieval task where we see the retrieval accuracy degrading as we extend to longer window size.

\end{enumerate}
}

\ignore{
This is encouraging since the results suggest the models still preserve relatively well their ability
in their original context window sizes even they are fine-tuned on much larger context windows.
This also suggests that the fact that Position
Interpolation puts position encodings into a much narrower region seem does not hurt
the modeling quality significantly.
}

\begin{table}[]
    \centering
    \begin{tabular}{cccccccccc}
       \toprule
       Model Size & Context Window & Fine-tune on & BoolQ & PIQA & Race-M & Race-H & WinoGrande \\
       \midrule
       7B & 2048 & None & 76.1 & 78.9 & 55.7 & 42.2 & 69.6 \\
       \midrule
       7B & 8192 & Pile & 73.2 & 78.2 & 53.8 & 41.7 & 69.0 \\
       7B & 16384 & Pile & 69.8   & 77.6 & 53.3 & 40.9  & 67.8 \\
       7B & 32768 & Pile & 64.7   & 77.2 & 50.1 & 39.6  & 66.9 \\
       7B & 8192 & RedPajama & 75.5 & 77.4 & 54.5 & 41.5 & 68.1 \\
       \midrule
       33B & 2048 & None & 81.6   & 80.2 & 61.1 & 45.9 & 76.2 \\
       \midrule
       33B & 8192 & Pile & 80.2   & 80.7 & 60.2 & 45.7 & 75.9 \\       
       \bottomrule
    \end{tabular}
    \caption{\small Zero-shot performance on a subset of LLaMA Benchmarks. Models extended by Position Interpolation comparable performance as the original models, except for BoolQ dataset that may require models to pay close attention to word ordering in a short reference paragraph.}
    \label{tab:llama_original}
\end{table}

\subsection{Long Document Summarization}

In this task, we evaluate our models' performance on the long document summarization task. 
In particular, we consider the GovReport \citep{huang-etal-2021-efficient} dataset, which contains 17457 documents for training and 972 documents for evaluation.
Each document comes with a human generated summary.
We truncate all input documents to their first 15000 tokens.

We fine-tune the LLaMA models extended with Position Interpolation with a context window
of 16384. Note the rescaling of position indices are still required during this fine-tuning step.
We first format the raw document using the prompt template in Figure~\ref{fig:summ-format}, and then concatenate the prompt
with the ground-truth summary (truncate to 1000 tokens) associated with each document. 
We fine-tune the model using the next token prediction task with the above setup for 10 epochs.
The losses from the input prompt proportion of training examples are excluded during our fine-tuning.

We use a generation temperature of 0.5 and $\text{top}_p = 0.95$ as our inference parameter to generate a summarization of each document in the test set. 
The final output is truncated at 1000 tokens.
We used the ROUGE-1/ROUGE-2/ROUGE-L scores \citep{lin-2004-rouge} as the evaluation metrics to evaluate the models' outputs vs the ground-truth summaries.

In Table~\ref{table:govreport} we report our evaluation results. 
We have also included results from two baselines in existing SCROLLS Leaderboard \citep{shaham-etal-2022-scrolls, ainslie2023colt5}. 
In general, we have obtained competitive R1 score among other models with minimal tuning of hyper-parameters. 
This result suggests our models with 16384 context window can effectively handle the long document summarization task.
\begin{figure}[th]
	\texttt{\noindent
        Read the following article and then summarize it. \\
        \# .... Document goes here \\
        Now summarize the above article. \\
        Summary:
    }
	\caption{\small Input format for long doc summarization.}
	\label{fig:summ-format}
\end{figure}

\begin{table}[]
    \centering
\begin{tabular}{ccccc}
\toprule
\multicolumn{2}{c}{Model} & \multicolumn{3}{c}{Evaluation Score}\\
Model & Context Window & ROUGE-1 & ROUGE-2 & ROUGE-L \\
\midrule
CoLT5 Base \citep{ainslie2023colt5} & 16K & 58.7 & 29.6 & 31.4 \\
CoLT5 XL \citep{ainslie2023colt5} & 16K &  61.3 & 32.2 & 33.8 \\
\midrule
LLaMA-7B Extended & 16K &  60.0 & 28.0 & 29.5 \\
\bottomrule
\end{tabular}
\caption{\small ROUGE Score on GovReport Dataset.}
\label{table:govreport}
\end{table}

\section{Related Work}

{\noindent \bf Retrieval-augmented LLM.} One line of work extends LLMs by augmenting it with retrieval modules which fetch related documents and include the retrieval results into the input context of an LLM \citep{karpukhin2020dense, guu2020realm, izacard2022atlas, jiang2022retrieval, khattab2021relevance, Santhanam2022colbertv2}. Our work is complementary to these works as our extended context window allows more documents being included in the input. In addition, with an unmodified attention mechanism and model architecture, our method may be more versatile as it can natively handle tasks beyond retrieval oriented ones, such as long document summarization, few-shots learning, etc.

{\noindent \bf Recurrent Transformers and Memory Transformers.} Several works add memory capabilities to Transformers through recurrence, which increase the models’ capability of handling very long sequences \citep{bulatov2022recurrent, wu2020memformer, dai2019transformerxl, wu2022memorizing, martins2021inftyformer, mu2023learning}. One limitation of these works is that they only allow attending to a lossy compressed version of past inputs. \citet{mu2023learning} suggested that this may prevent models from remembering specific details in the past inputs. In contrast, our work allows attending to all previous tokens, preserving all details without compression, albeit with higher inference costs. \citet{mohtashami2023landmark} proposed landmark attention which allows full random access to any chunk of the input through introducing landmark tokens. Our work allows full access of the entire input through unmodified attention, which may be useful for tasks such as summarization.

{\noindent \bf Approximated Multi-head Attention.} There is a large body of research that focuses on decreasing the memory and computational complexity of the multi-head attention (MHA) mechanism through approximation or sparsification \citep{child2019generating, zaheer2020bigbird, beltagy2020longformer, wang2020linformer, choromanski2021rethinking, kitaev2020reformer, ren2021combiner}. Although not the focus of this work, as these methods are not used in LLaMA \citep{touvron2023llama}, we note that our method is compatible with most of them since our changes are restricted to position encodings, and not attention mechanisms.

{\noindent \bf Length Extrapolation.} A recent line of research aims to train Transformers models on short sequences and inference on longer \citep{press2022train, sun2022lengthextrapolatable, haviv2022transformer}. However, these methods have not been applied in some of the largest language models such as LLaMA \citep{touvron2023llama}, or OPT \citep{zhang2022opt}.  This has prevented them from enabling length extrapolation of many pre-existing pre-trained language models. Our work focuses on extending existing LLMs, which can save substantial pre-training costs. In addition, our method preserves the quality of the original models, even for small context window tasks, since it does not deviate far from existing definitions of position encoding or attention mechanisms.

{\noindent \bf Interpolation.} The most related technique to ours is proposed by \citet{dosovitskiy2021an} in their work on Vision Transformers, where the authors proposed to linearly interpolate learnt position embeddings to support higher resolution, which translates to an increased number of input embeddings, in the fine-tuning stage. The interpolated position embedding weights are used as initialization in the fine-tuning process for the newly added positions. Our work differs from their work in several ways (1) Instead of interpolating position embeddings, our method interpolates position indices, which is more suitable for RoPE like position encodings and may require less training since no trainable parameters are added. (2) We report successful results of extending the context window to 32 times while \citet{dosovitskiy2021an} explored up to 4 times. Our results extend theirs in exploring the upper limit of context window extension via interpolation. (3) We evaluated and confirmed the effectiveness of Position Interpolation for extending context windows for language models.

\ignore{
{\noindent \bf Extending Large Language Model's Context Window} Recently there have been a few works trying to extend LLM's context window length w/o re-running pre-training which can be prohibitively expensive. \cite{ratner2022parallel} directly extend LLM's context window length at inference time by chunking the input to smaller windows; \cite{mohtashami2023landmark} extended to 32k size by finetuning "landmark" tokens. \cite{kazemnejad2023impact} proposed to remove positional encoding completely to bypass the window size limit. Our work aims at extending LLM's context window without any assumption on specific downstream tasks; in other words, we shoot for extending context length while preserving LLM's general-purpose capabilities.  
}
We believe our results, in conjunction with \citep{dosovitskiy2021an}, provide empirical evidence on Transformer's remarkable ability of handling significantly longer sequences beyond training. 
Further, we conjecture that a method similar to theirs is directly applicable in LLMs with learnable position embeddings such as OPT \citep{zhang2022opt} and we plan to investigate this in the future.

\section{Conclusions}
Position Interpolation can effectively extend LLaMA models' context window to be significantly larger, using minimal fine-tuning.
The extended models are fully capable to perform a variety of tasks on the extended context windows, and preserve its original 
ability relatively well for tasks within the original extended models, making them good choices of generic language models 
for both long and short input prompts.
Further, models extended by Position Interpolation can reuse most pre-existing infrastructure and optimization, making this method
attractive in many practical applications.
We believe that Position Interpolation is a general method that could be apply to other 
types of position encodings, which can allow extension for more types of LLMs, and we plan to investigate in such directions in the near future.

\section*{Acknowledgements}
We thank Mike Lewis for his input on evaluation.

\bibliography{reference}
\bibliographystyle{iclr}

\def\dd{\mathrm{d}}

\newpage
\begin{center}
\textbf{\Large{Appendix}}    
\end{center}

\appendix 
\section{Proof}
\label{sec:proof}
\interp*
\begin{proof}
Using Taylor expansion, we have:
\begin{eqnarray}
    a(s_1) &=& a(s) + a'(s)(s-s_1) + \frac12 a''(\xi_1)(s-s_1)^2 \label{eq:xi1} \\
    a(s_2) &=& a(s) + a'(s)(s-s_2) + \frac12 a''(\xi_2)(s-s_2)^2 \label{eq:xi2}
\end{eqnarray}
where $\xi_1\in [s_1,s]$ and $\xi_2\in [s,s_2]$. Multiplying Eqn.~\ref{eq:xi1} with $s-s_2$ and Eqn.~\ref{eq:xi2} with $s-s_1$ and subtract, we get:
\begin{equation}
    a(s) - a_{\mathrm{linear}}(s) = R(s) := -\frac{(s-s_1)(s-s_2)}{2(s_1-s_2)}\left[a''(\xi_1)(s-s_1) - a''(\xi_2)(s-s_2)\right] 
\end{equation}
Now we bound the second order derivative $a''(s)$. Note that for any complex number $x$, $|\mathrm{Re}(x)| \le |x|$ so we have:
\begin{eqnarray}
    |a''(s)| &\le& \sum_{j=0}^{d/2-1} |h_j| |\phi''_j(s)| \le \sum_{j=0}^{d/2-1} |h_j| \theta^2_j \\
    &\le& \left(\max_j |h_j|\right) \sum_{j=0}^{d/2-1} c^{-4j/d} = \left(\max_j |h_j|\right) \frac{1}{1 - c^{-4/d}}
\end{eqnarray}
Note that when $x < 0$ and $c > 1$, $c^x \le 1 + x\ln c$, therefore $c^{-4/d} \le 1 - 4/d\ln c$ and we have:
\begin{equation}
    \frac{1}{1-c^{-4/d}} \le \frac{1}{4/d \ln c} = \frac{d}{4\ln c}
\end{equation}
So 
\begin{equation}
    |a''(s)| \le \left(\max_j |h_j|\right)\frac{d}{4\ln c} =: M
\end{equation}
Let the above bound to be $M$, we have:
\begin{equation}
    |R(s)| \le \frac{(s-s_1)(s_2-s)}{2(s_2-s_1)}\left[M(s-s_1) + M(s_2-s)\right] = \frac{M}{2}(s-s_1)(s_2-s)
\end{equation}
As a result:
\begin{equation}
    |a(s) - a_{\mathrm{linear}}(s)| = |R(s)| \le d\left(\max_j |h_j|\right) \frac{(s-s_1)(s_2-s)}{8\ln c}
\end{equation}
\end{proof}

\section{Visualization of quantities in extrapolation bound}
\label{sec:visualization}
As shown in Eqn.~\ref{eq:extra-bound}, the extrapolation bound contains the term  $B(s) := \sum_{k=0}^{d/2-1} |A_{k+1}(s)|$ where $A_k(s) := \sum_{j=0}^{k-1} e^{\di s\theta_j}$. Here we check how large the bound is. We use $\theta_j = c^{-2j/d}$ with $c = 10000$ and $d = 4096 / 32 = 128$ (LLaMA-7B setting), and Fig.~\ref{fig:b_bound}  shows that $B(s)/d$ almost always larger than $1$ and in many places it is much larger than $1$. 

\begin{figure}
    \centering
    \includegraphics[width=\textwidth]{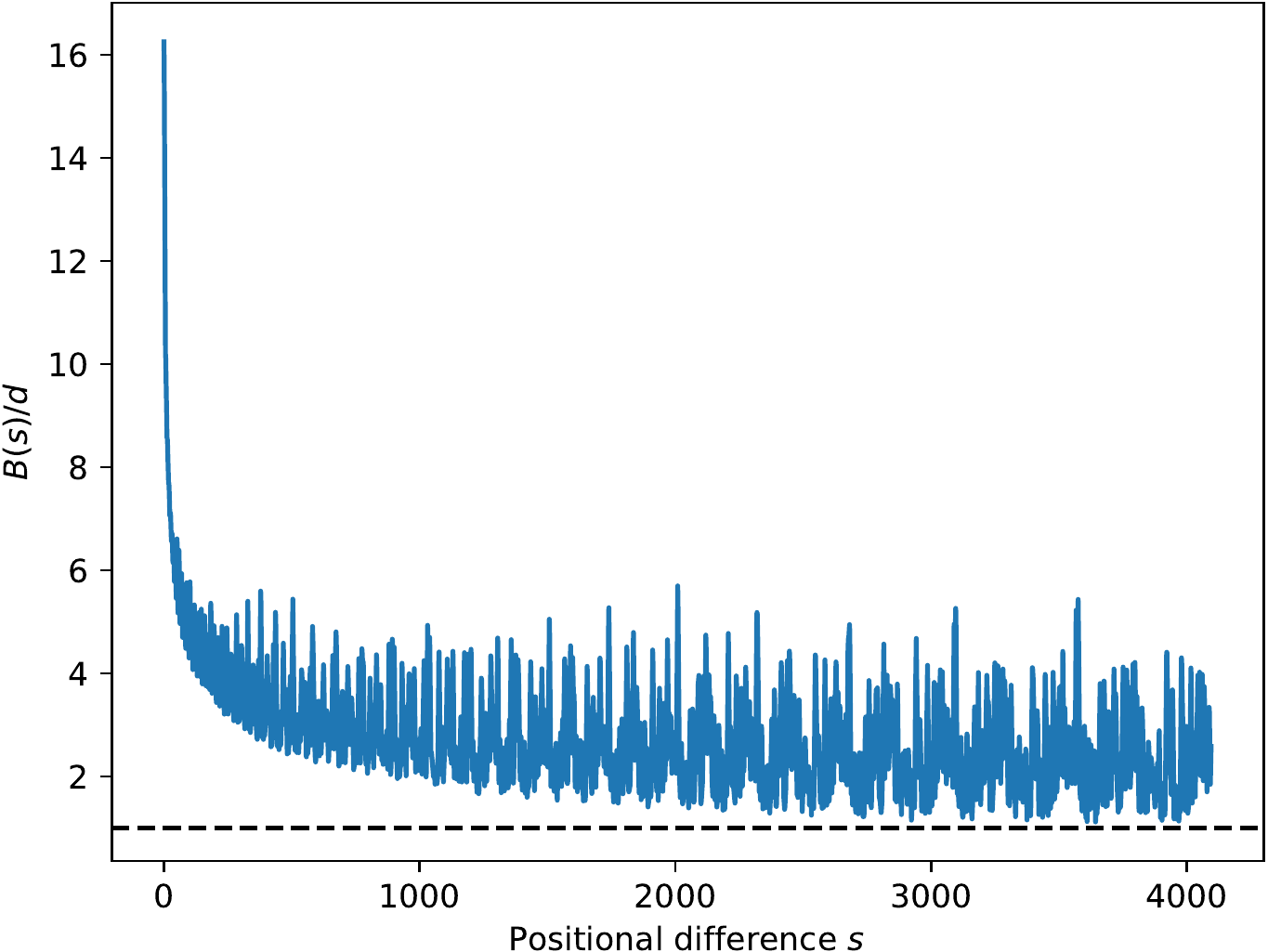}
    \caption{\small The bound $B(s) / d$ decays with $s$. While the bounds goes down with large positional difference $s$, numerically $B(s)/d \ge 1$ and at many $s$ much larger than $1$ (the dotted horizontal line). Please check Appendix~\ref{sec:b-bound} for the source code used to draw the figure.}
    \label{fig:b_bound}
\end{figure}

\section{Code}
\subsection{Code for Fig.~\ref{fig:interp-vs-extrap}}
\label{sec:code-inter-extra}
{\small
\begin{verbatim}
# build basis function
d = 4096 // 32
theta = 10000
# Frequency computation, 
freqs = 1.0 / (theta ** (torch.arange(0, d, 2)[: (d // 2)].float() / d))

# construct basis function 
L = 2048

x = torch.zeros(L)
x[:L] = torch.arange(0, L)

# basis functions
xfreq = torch.outer(x, freqs)

y = torch.randn(x.shape[0])

# do linear regression
X = torch.cat([xfreq.sin(), xfreq.cos()], dim=1)

eps = 0.000
coeffs = torch.linalg.solve(X.t() @ X + torch.eye(X.shape[1]) * eps, X.t() @ y)

x2 = torch.arange(0, 2*L)
xfreq2 = torch.outer(x2, freqs)
X2 = torch.cat([xfreq2.sin(), xfreq2.cos()], dim=1)

y2 = X2 @ coeffs

x3 = torch.arange(25, 75, 0.125)
xfreq3 = torch.outer(x3, freqs)
X3 = torch.cat([xfreq3.sin(), xfreq3.cos()], dim=1)

y3 = X3 @ coeffs

plt.figure(figsize=(16,5))

plt.subplot(1, 3, 1)
plt.plot(x2[:L], y2[:L], "r")
plt.scatter(x, y)
plt.ylabel("attention score $a(s)$")
plt.xlabel("Positional difference $s$")

plt.subplot(1, 3, 2)

plt.plot(x2, y2, "r")
plt.scatter(x, y)
plt.axvline(L, color="k", linestyle="--", linewidth=0.5)

plt.title("Effect of Extrapolation")
plt.xlabel("Positional difference $s$")


plt.subplot(1, 3, 3)
plt.plot(x3, y3, "r")
for i in range(25,75):
    plt.axvline(i, color="k", linestyle="--", linewidth=0.5)
plt.title("Effect of Interpolation")
plt.xlabel("Positional difference $s$")
plt.show()
\end{verbatim}
}

\subsection{Code for Fig.~\ref{fig:b_bound}}
\label{sec:b-bound}
{\small
\begin{verbatim}
L = 2048
x = torch.arange(0, 2*L)
d = 4096 // 32
theta = 10000
freqs = 1.0 / (theta ** (torch.arange(0, d, 2)[: (d // 2)].float() / d))

xfreq = torch.outer(x, freqs)

mags = (xfreq.sin().cumsum(dim=1).pow(2) + xfreq.cos().cumsum(dim=1).pow(2)).sqrt()

plt.plot(mags.sum(dim=1)/d)
plt.axhline(1.0, color='k', linestyle="--")
plt.xlabel("Positional difference $s$")
plt.ylabel("$B(s)/d$")
plt.show()
\end{verbatim}
}

\end{document}